\documentclass{article}

\usepackage[final]{neurips22}

\usepackage[utf8]{inputenc} % allow utf-8 input
\usepackage[T1]{fontenc}    % use 8-bit T1 fonts
\usepackage{hyperref}       % hyperlinks
\usepackage{url}            % simple URL typesetting
\usepackage{booktabs}       % professional-quality tables
\usepackage{amsfonts}       % blackboard math symbols
\usepackage{nicefrac}       % compact symbols for 1/2, etc.
\usepackage{microtype}      % microtypography
\usepackage{xcolor}         % colors
\usepackage{multirow}

\usepackage{amsmath}
\usepackage{amssymb}
\usepackage{mathtools}
\usepackage{amsthm}

\usepackage{url}
\usepackage{algorithm}
\usepackage{algorithmic}
\usepackage{times}
\usepackage{epsfig}
\usepackage{graphicx}
\usepackage{verbatim}
\usepackage{diagbox}

\usepackage{authblk}

\usepackage{lipsum}

\newcommand{\scriptL}{\mathcal{L}}

\newcommand{\scriptX}{\mathcal{X}}
\newcommand{\scriptY}{\mathcal{Y}}
\newcommand{\scriptK}{\mathcal{K}}

\DeclarePairedDelimiter\abs{\lvert}{\rvert}%
\DeclarePairedDelimiter\norm{\lVert}{\rVert}%

\newcommand\numberthis{\addtocounter{equation}{1}\tag{\theequation}}

\newtheorem{proposition}{Proposition}
\newtheorem{definition}{Definition}
\newtheorem{theorem}{Theorem}
\newtheorem{lemma}{Lemma}

\newtheorem{remark}{Remark}

\usepackage[capitalize,noabbrev]{cleveref}

\newcommand\blfootnote[1]{%
  \begingroup
  \renewcommand\thefootnote{}\footnotetext{#1}%
  \addtocounter{footnote}{-1}%
  \endgroup
}
\newcommand*{\email}[1]{%
    \normalsize\href{mailto:#1}{#1}\par
    }
\theoremstyle{plain}

\title{A Neural Pre-Conditioning Active Learning Algorithm to Reduce Label Complexity}

\author{%
\textbf{Seo Taek Kong}$^{1,*,\dagger}$ \quad \textbf{Soomin Jeon}$^{2}$ \quad \textbf{Dongbin Na}$^3$ \quad \textbf{Jaewon Lee}$^3$ \quad \textbf{Hong-Seok Lee}$^3$ \quad \textbf{Kyu-Hwan Jung}$^{4,*,\dagger}$ 
    \\
$^1$University of Illinois, Urbana-Champaign \quad $^2$Dong-A University \quad $^3$VUNO Inc. \quad $^4$Sungkyunkwan University
}

\begin{document}

\maketitle
% \icmlkeywords{Active Learning, Neural Tangent Kernel, Semi-Supervised Learning}

% \vskip 0.3in
% ]

% this must go after the closing bracket ] following \twocolumn[ ...

% This command actually creates the footnote in the first column
% listing the affiliations and the copyright notice.
% The command takes one argument, which is text to display at the start of the footnote.
% The \icmlEqualContribution command is standard text for equal contribution.
% Remove it (just {}) if you do not need this facility.

% \printAffiliationsAndNotice{}  % leave blank if no need to mention equal contribution
% \printAffiliationsAndNotice{\icmlEqualContribution} % otherwise use the standard text.

\begin{abstract}
    Deep learning (DL) algorithms rely on massive amounts of labeled data. Semi-supervised learning (SSL) and active learning (AL) aim to reduce this label complexity by leveraging unlabeled data or carefully acquiring labels, respectively. 
    In this work, we primarily focus on designing an AL algorithm but first argue for a change in how AL algorithms should be evaluated. 
    Although unlabeled data is readily available in pool-based AL, AL algorithms are usually evaluated by measuring the increase in supervised learning (SL) performance at consecutive acquisition steps. 
    Because this measures performance gains from both newly acquired instances and newly acquired labels, we propose to instead evaluate the label efficiency of AL algorithms by measuring the increase in SSL performance at consecutive acquisition steps. 
    After surveying tools that can be used to this end, we propose our neural pre-conditioning (NPC) algorithm inspired by a Neural Tangent Kernel (NTK) analysis.
    Our algorithm incorporates the classifier's uncertainty on unlabeled data and penalizes redundant samples within candidate batches to efficiently acquire a diverse set of informative labels. 
    Furthermore, we prove that NPC improves downstream training in the large-width regime in a manner previously observed to correlate with generalization. Comparisons with other AL algorithms show that a state-of-the-art SSL algorithm coupled with NPC can achieve high performance using very few labeled data.
    % Deep learning (DL) algorithms rely on massive amounts of labeled data, and semi-supervised learning (SSL) and active learning (AL) algorithms have been designed to reduce this label complexity by leveraging unlabeled data or carefully acquiring labels. In this work, we primarily focus on designing an AL algorithm but first argue for a change in how AL algorithms should be evaluated. Although unlabeled data is readily available in pool-based AL, experimental evaluations had typically compared the performance of supervised learning (SL) as labels are incrementally acquired. Instead we argue that SSL performance should be used to evaluate AL algorithms to measure their label efficiency. Focusing on this objective and after surveying tools that can be used to this end, we propose a neural pre-conditioning (NPC) algorithm, based on a neural tangent kernel (NTK) analysis, that valuates unlabeled data based on how they would contribute upon inclusion to the training set. Our algorithm uses uncertainty information captured by the networks gradients, but differs from other gradient-based AL algorithms in how diversity is enforced.
    % Furthermore, we prove that NPC improves downstream training landscape in the NTK regime, whose properties are known to correlate with generalization. Comparisons with other AL algorithms show that a state-of-the-art SSL algorithm coupled with NPC can achieve high performance using very few labeled data.
\end{abstract}

\section{Introduction}
% \footnote{$\dagger$ This work was submitted while the authors were at Vuno Inc.}
Active learning (AL) describes the setting where a model can interact with a dedicated annotator and query for labels.
\blfootnote{* This work was submitted while the authors worked at Vuno Inc.\hfill}
This is in contrast to passive learning where labels are acquired randomly.
% instances are acquired.
% A passive algorithm randoml
In pool-based AL, both unlabeled $\mathcal{Z}_U = \scriptX_U$ and labeled data $\mathcal{Z}_L = (\scriptX_L, \scriptY_L)$ are available for training, and labels are incrementally acquired by querying $\abs{\scriptY^s} = Q$ labels at each query step $s$ until a labeling budget $B$ is met. 
\blfootnote{$^\dagger$ Correspondence to \email{skong10@illinois.edu, khwanjung@skku.edu}}
% and \email{khwanjung@skku.edu}. 
% \email{skong10@illinois.edu} and \email{khwanjung@skku.edu}. }
The question we aim to answer in this work is: ``Given images $\scriptX$ and a labeling budget $B$, what's the maximum performance that can be achieved?''
Traditionally, AL algorithms have been evaluated by measuring downstream supervised learning (SL) performance, i.e. training on $\mathcal{Z}_L$, but we argue that downstream semi-supervised learning (SSL) performance, i.e. training on $\mathcal{Z} = \mathcal{Z}_L \cup \mathcal{Z}_U$, is a better benchmark for the following reasons.
Supervised learning performance as an evaluation metric for AL fails to extricate performance gains from newly-acquired labels $\scriptY^{s+1}$ from the influence of newly-acquired samples $\scriptX^{s+1}$.
In contrast, SSL trains a model on $\mathcal{Z}$ and the only difference of datasets $\mathcal{Z}^{s+1} - \mathcal{Z}^{s}$ after subsequent query steps is the labels $\scriptY^{s+1}$, and the respective performances reveal gains from newly-acquired labels.
While it has been shown that many AL algorithms outperform passive learning (PL) with respect to instance-label efficiency, we see that some AL algorithms in fact under-perform PL with respect to this criterion that measures label efficiency.
% , we see that some AL algorithms in fact under-perform passive learning (PL) with respect to the latter.
% Another benefit of using SSL as an evaluation metric, when applied appropriately, for AL algorithms is that in retrospect we observed a larger distinction in downstream classification performance between AL and passive learning (PL).
% Enhancements from applying AL appeared to be minor when evaluated using SL, where large query sizes are often used to achieve meaningful performance.
% Modern SSL algorithms have the potential to achieve remarkable performance, under the ideal condition that classes are balanced, using only a few labels. 
% With the use of SSL, we were able to highlight benefits of AL in the low-label regime. 
% limitation of using SL performance as a benchmark for AL algorithms' label efficiency is that performance gains per-label-acquired remain minor relative to passive learning (PL).
% as evaluated by downstream SL performance remain minor relative to passive learning (PL).
% This may be attributed to the large numbers of labels necessary to attain meaningful accuracy for SL, which would indicate that such evaluation metrics may have been assessing near-asymptotic label complexity where the enhancements achievable by AL become relatively small.
% Modern SSL algorithms achieve remarkable performance with only few labels when classes are balanced, and {\color{red}Finish paragraph!}

Simply replacing SL with SSL can introduce new problems in evaluating the label-efficiency of AL algorithms.
% Issues arise when simply replacing SL with SSL.
While in principle more data should always be better, SSL performance deteriorates significantly when the labeled set's classes are imbalanced \citep{lee2021abc,darp}.
In AL, the number of images corresponding to each class cannot be observed prior to labeling, and constructing a balanced labeled set would require discarding majority classes retrospectively.
Class imbalance had not been as problematic when benchmarking AL algorithms with SL performance because large query sizes ultimately yields class distributions closer to uniform.
% when evaluating AL through downstream SL performance because large query sizes were used, where .
% , since algorithms designed for the latter setting are significantly affected by class imbalance .
% A notable limitation of benchmarks assessing SSL algorithms is that the labeled set is assumed to be balanced. 
% labels are acquired retrospectively after and level of class imbalance cannot be observed prior to labeling, and constructing a balanced labeled set would require discarding majority classes.
% In practice, images are often labeled retrospectively and constructing a balanced labeled set would require discarding majority classes. 
% Passive learning and other AL algorithms therefore most likely results in class imbalance.
% While in principle more data should always be better, SSL performance deteriorates significantly when the labeled set's classes are imbalanced \citep{lee2021abc,darp}.
% algorithms trained on more data if classes are not balanced which is the case when labels are actively acquired.
We address this issue by adopting the recently-proposed distribution re-alignment method \citep{darp} applied to a widely used SSL algorithm called FixMatch \citep{fixmatch}.
% fixing the downstream SSL algorithm used to evaluate the label-efficiency of AL algorithms with a state-of-the-art algorithm FixMatch \citep{fixmatch} combined with DARP \citep{darp} designed to handle class-imbalance in SSL.
% Modern SSL algorithms can achieve high performance with few labels, and we show that the strength of AL over PL is evident in this low-label regime.

% As a solution, we utilize the recently-developed DARP algorithm \citep{darp} designed to handle class-imbalance in SSL.
% The use of DARP mitigates performance degradation when the labeled set is enlarged, and we show that AL can reap great benefits over PL in this practical setting.

Having motivated downstream SSL performance as an evaluation metric for AL, we seek to maximize performance with respect to a labeling budget $B$.
We propose a neural pre-conditioning (NPC) algorithm that builds on prior work and addresses problems noticed in literature.
Our algorithm uses the Gram matrix of the model's gradients with respect to parameters.
Gradients of the loss function have been used by \citet{EGL,badge} in the context of AL to model a classifier's uncertainty about unlabeled samples, whereas we use gradients of the classifier's outputs.
Because of the difference, the embeddings used by NPC capture uncertainty by comparing the direction of gradients with other possibly more certain samples' in addition to their magnitudes.
Moreover, our algorithm operates in the batch-mode setting where the collective importance of candidate samples is measured together.
Lastly, we show that the landscape of downstream SSL is improved when supervised on labeled data selected by NPC which is why we name it neural pre-conditioning.

This paper is structured as follows.
Section \ref{sec:Problem} formally describes why comparing downstream SL performance after consecutive acquisition steps fails to measure the label-efficiency of AL algorithms, and proposes to consider downstream SSL performance as a fair evaluation metric.
Section \ref{sec:Method} lays out the observations made in prior works that motivate the proposed algorithm before proceeding to stating the algorithm and how it addresses these concerns.
Section \ref{sec:Discussion} addresses the last pre-conditioning property of the algorithm and describes potential benefits to a randomized search procedure invoked by NPC.
% provides a feedback system of the overall problem setting when using the proposed algorithm as a label selection policy.
% This view complements how gradient embeddings are used to capture the model's uncertainty.
Lastly, Section \ref{sec:experiments} presents experiments that show how the proposed algorithm enhances downstream SSL performance, and highlights how some AL algorithms are not as effective in our proposed setting.
% have have been ranked improperly due to their evaluation being with respect to downstream SL performance.

\section{Problem Setting and Related Work}\label{sec:Problem}
\subsection{Active Semi-Supervised Learning}
To measure performance gains from only newly acquired labels, we alternate between applying AL to acquire labels and training a classifier on newly acquired data using SSL, where at first a small set of labels $\abs{\scriptY_L^0} = Q_0$ with balanced classes is assumed.
% For training, a state-of-the-art SSL algorithm FixMatch \citep{fixmatch} equipped with pseudo-label refinement (DARP, \citet{darp}) is used to handle class imbalance.
At each query step $s$, the classifier queries for $Q$ labels $\mathcal{Y}^s$ corresponding to samples ${\mathcal X}^s$ from the remaining unlabeled pool $\mathcal{X}_U$.
A state-of-the-art SSL algorithm named FixMatch \citep{fixmatch} with a pseudo-label refinement procedure (DARP, \citet{darp}) is used to handle class imbalance when training after subsequent query/acquisition steps. 
% corresponding to images $\scriptX_u^* \subset \scriptX_U$ selected.
We refer to the above procedure as active semi-supervised learning (ASSL) following \citep{Henneke_agnostic}, and remark that the term has been used to refer to different procedures \citep{CEAL}.
Our ASSL setting closely follows standard AL benchmarks \citep{badge,coreset,ShalevSchwartz} with the difference being that SSL, instead of SL, is invoked for training.

We explain why comparing downstream SSL, instead of SL, performance at consecutive acquisition steps is a better evaluation scheme when measuring the label-efficiency of AL algorithms.
Consider two fully-trained classifiers at consecutive acquisition steps $s$ and $s+1$.
The performance difference of downstream SL (ASL) performance is given by
\begin{equation}
 \mathbb{P}\left(\hat{y}\left(x_{test}; \scriptX_L \cup \scriptX^{s+1}, \scriptY_L \cup \scriptY^{s+1}\right) \neq y_{test}\right) - 
  \mathbb{P}\left(\hat{y}\left(x_{test}; \scriptX_L, \scriptY_L \right) \neq y_{test}\right) ,
\end{equation}
where $\hat{y}\left(x_{test}; \cdot\right)$ is the prediction of a classifier trained on data $\cdot$.
In contrast, the difference of downstream SSL performance 
\begin{equation}
    % \min_{\scriptY^* \subset \scriptY_U} \mathbb{P}\left(\hat{y}\left(x_{test}; \scriptX_L \cup \scriptX_U, \scriptY_L \cup \scriptY^*\right) \neq y_{test}\right) 
    \mathbb{P}\left(\hat{y}\left(x_{test}; \scriptX_L \cup \scriptX_U, \scriptY_L \cup \scriptY^{s+1}\right) \neq y_{test}\right) - \mathbb{P}\left(\hat{y}\left(x_{test}; \scriptX_L \cup \scriptX_U, \scriptY_L \right) \neq y_{test}\right) 
\end{equation}
measures the gain from only newly-acquired labels $\scriptY^{s+1}$.
While ASL is affected by both newly-acquired images and labels, ASSL extricates the two and reveals performance gains from only the newly-acquired labels.
% While ASSL reveals performance gains from only the newly-acquired labels
% The difference in classification accuracies reveal performance gains from only the newly-acquired labels $\mathcal{Z}^{s+1} - \mathcal{Z}^s = \scriptY^{s+1}$, extricating influence due to images.
The above description motivates one reason to consider ASSL, but it is clear that AL can be applied to improve SSL performance as in \citep{MixMatchAL}.

Despite its importance, we believe two main hurdles restrained prior works to consider ASSL.
SSL algorithms have seen great advances only recently \citep{berthelot2019mixmatch,fixmatch} and their full strengths simply weren't available.
On CIFAR-10, state-of-the-art SSL algorithms presented with as few as 40 labels are now able to match full-supervision where all labeled data is used.
% all labeled data, when presented with as few as 40 labels.
% where now their performances match fully-supervision, using all labeled data, with as few as 40 labels; 
Second, SSL performance degrades significantly when the class distribution of labeled data is imbalanced.
In AL, a-priori enforcing balanced classes is impossible because labels are unknown and discarding majority classes (under-sampling) wastes what was spent to acquire the labels.
When the labeled set's classes are highly imbalanced, pseudo-labels generated by SSL algorithms are even more-so imbalanced \citep{darp}.
By adopting a pseudo-label refinement process, we rectify performance degradation caused by class imbalance and are able to achieve increasing performances when incrementally acquiring more labels.

\subsection{Related Work}
% Due to space constraints, we describe only DL-based AL algorithms and recommend \citep{Henneke} for a theoretical treatment of AL.
% See \citep{survey} for a comprehensive treatment of DL-based AL.
\subsubsection{Active Learning}
Only DL-based AL algorithms are surveyed, where version-space approaches become trivial due to their expressive power \citep{badge}.
A fully-trained classifier is used to query for labels of samples from a pool of unlabeled data $\scriptX_U$.
Many AL algorithms can be characterized by how they valuate each candidate batch $\scriptX \subset \scriptX_U$, where labels corresponding to the batch maximizing some scoring function $v\left(\scriptX\right)$ are acquired.
% design for each candidate batch $\scriptX$, where the query strategy acquires labels corresponding to the batch that maximizes $v$.
Among the earliest algorithms, the uncertainty-based algorithms developed in \citep{uncertainty} score each sample $x_i$ using the classifier's margin $v_i = \min_{y' \neq \hat{y}} f_{\theta} (x_i; \hat{y}) - f_{\theta} (x_i; y')$, or entropy $H\left(\sigma \left(f(x_i ; \cdot)\right)\right)$ where $\sigma$ is the softmax function.
Because DNNs are often mis-calibrated and their softmax probabilities are not a good proxy for uncertainty \citep{calibration}, a line of work \citep{batchbald} use Bayesian neural networks \citep{bayesian_dnn}.
\citet{ShalevSchwartz} computes the $\mathcal{H}$-divergence \citep{Hdivergence} resulting from hypothetical inclusions of unlabeled samples to the labeled set, and selects those that best aligns the distributions underlying labeled and unlabeled sets.
% constructs the labeled set such that the unlabeled and labeled distributions are aligned with respect to $\mathcal{H}$-divergence widely used in domain adaptation \citep{Hdivergence}.
\citet{coreset} pose each query step as a core-set selection problem and finds an approximate solution.
EGL \citep{EGL} and BADGE \citep{badge} use the gradients of a loss on unlabeled samples as proxies for uncertainty.
The former queries for samples that maximize the gradient norm, while the latter diversifies gradient embeddings using k-means++.
% selects samples that diversify gradient embeddings based on the premise that networks trained with GD variants have uncertainty information embedded in their gradients.
% argues that networks trained with GD-based optimizers have uncertainty information embedded in their gradients, and selects samples that diversify gradient embeddings.

\subsubsection{Semi-supervised Learning}
Modern SSL algorithms utilize unlabeled samples and add a consistency loss to act as a regularization in addition to the standard supervision loss.
FixMatch \citep{fixmatch} is a state-of-the-art algorithm that combines and simplifies a sequence of developed SSL methods \citep{pseudolabel,pi_model,mean_teacher,berthelot2019mixmatch} by generating pseudo-labels with weakly-augmented samples.
\citet{darp,lee2021abc} observe that pseudo-labels generated by related algorithms \citep{berthelot2019mixmatch,fixmatch,remixmatch} are severely imbalanced when the classifier is trained on imbalanced data, thereby degrading performance. 
In AL, it is impossible to ensure balanced classes in either the labeled or unlabeled sets and the same problem persists.
For our problem setting, we use FixMatch-DARP \citep{darp} where pseudo-labels are post-processed such that their class distribution matches a target distribution.
Because for general purposes it is impractical to assume knowledge of class distribution underlying unlabeled samples, we set this target as the uniform distribution. 
% generates pseudo-labels as the classifier's predictions on weakly-augmented unlabeled samples.
% \section{Problem Setting}

% {\color{red}
% SSL under Class Imbalance:
% Also describe DARP and class imbalance.
% }

\section{Motivations and Method}\label{sec:Method}
% \subsection{Motivations}
\subsection{Notations}
% \textbf{Notations:}
A classifier's output layer (preceding softmax) is denoted as $f_{\theta}$, and the gradients with respect to its parameters as $\nabla f_{\theta}$.
We often drop the subscript $\theta$ and leave it otherwise for emphasis.
For simplicity of exposition, we describe our notations assuming a single class and note that this can easily be re-written following \citep{Arora19,zhu19,du19} for multi-class classification.
Given $N$ samples $\scriptX = \left\{x_1, \dots, x_N\right\}$ and $d$ parameters, the dimension of networks gradient is listed as $\nabla f_{\theta} \left(\scriptX\right) \in \mathbb{R}^{N \times d}$.
The Gram matrix $\scriptK_t \left(\scriptX, \scriptX'\right) := \nabla f_{\theta_t}\left(\scriptX\right) \nabla f_{\theta_t}^T\left(\scriptX \right)$ computed using parameters $\theta_t$ obtained after $t$ optimization (e.g. SGD) steps is also known as the empirical NTK \citep{Arora19}.

% {\color{red}
% Maybe discuss that we used the MSE loss, and provide connections to dynamics of SGD with NLL loss?
% Or do not specify until experiments, where we say we've done so based on our kernel analysis for training section.
% Actually: I don't think we need to mention this at all; instead at training, this is where the loss comes into play. As of now we've defined the kernel as the Gram matrix.
% }
% We focus on the mean-squared error (MSE) loss function $\scriptL\left(\theta\right) = \norm{f_{\theta} \left(\scriptX\right) - \scriptY}^2$.
% The ordinary differential equation corresponding to 

\subsection{Motivations}
\subsubsection{Uncertainty Embeddings}\label{sec:Uncertainty}
Motivated by the ubiquity of stochastic gradient descent (SGD) used to train deep neural networks, \citet{EGL,badge} use gradient embeddings $\nabla \scriptL \left(\theta; x, \hat{y}\right)$ to measure the uncertainty about a sample $x$ using a proxy label $\hat{y}$.
We adopt a similar view on gradients and use them to valuate unlabeled samples, except that our algorithm will make use of the network's gradients $\nabla f$ which is related to the loss gradients $\nabla \scriptL$ through the chain rule.
However, while the arguments used in above references are mainly based on the idea that uncertain samples cause large gradients $\nabla \scriptL$, this is not necessarily true for the network's gradients $\nabla f_{\theta}$.
% and the magnitude is of little interest.
Instead, the gradients' directions are additionally used to measure uncertainty.
% a pair of samples are evaluated
% gradients are related through their directions in the context of measuring uncertainty.
A network supervised on data including a labeled sample $x_l$ will be more-so certain on that sample than on an unlabeled sample $x_u$, and the product $\nabla f_{\theta_t} (x_l)^T \nabla f_{\theta_t} (x_u)$ being small indicates uncertainty about $x_u$, and in turn that $x_u$ should be queried for its label.
We show, after presenting our algorithm in Sec. \ref{sec:Algorithm}, that our selection criterion captures uncertainty information by comparing the gradient's direction with a confident reference vector evaluated at a labeled instance.
% , which is desirable given that we use $\nabla f$ instead of $\nabla \scriptL$.

\subsubsection{Batch-mode Operation}
Given a fixed labeling budget $\abs{\scriptY_L} \leq B$, a lower bound on the query size $Q$ is determined by how often the classifier can query the label oracle or worker.
When the worker is not to be disturbed, a large query size (e.g. $Q=B$) is necessary, and ideally an AL algorithm should attain higher performance when querying more often.
To avoid excessive numbers of queries, one of the most important traits of an AL algorithm is batch mode operation, valuating the collective importance of a candidate batch instead of its marginal elements.
% Generally, this means that it is desirable that an AL algorithm works in the batch setting $Q > 1$.
Early DL-based AL algorithms were myopic $v\left(\scriptX\right) = \sum_i v(x_i)$, meaning that their valuation of samples does not consider the collective value of candidate batches.
For example, max-margin is a myopic policy and queries redundant samples \citep{batchbald} when duplicates are present. 
% Batch-mode AL queries for a batch $\scriptX$ based on its collective importance, rather than the marginal importance of its elements $x \in \scriptX$.
Algorithms that operate in the batch setting prove critical as query size becomes large.

\subsubsection{Loss Landscape and Classification Performance}
Loss landscape has long been connected to generalization (classification) performance, one view being that critical points near flat minima are more robust to distribution shifts occurring between train and test sets \citep{asymmetric}.
Gradient steps in flat landscapes that do not take into account second order information for re-scaling inevitably take small steps, but it has been observed in \citep{fastSWA} that SGD continues to take large steps when applied to losses used in SSL.
Together these views suggest that an improved loss landscape for SSL would enhance generalization performance.

To this end, one of our considerations in designing an AL algorithm is to construct a training set so that the induced landscape exhibits properties positively correlated with generalization as discussed above.
% Flatness of a landscape can be inferred from the set of admissable step-sizes, or more generally the condition number determining convergence at the local landscape.
We show that in addition to the utilization of uncertainty information from network's gradients and diversity enforcement, another view for our objective is to select data that ameliorates downstream training. 
Because our algorithm improves the conditioning of downstream optimization problem, we name it Neural Pre-Conditioning (NPC).

% This was identified by Fast-SWA, but their method adds additional hyperparameters, and FixMatch and other SOTA SSL algorithms are designed for Mean Teacher instead of their SWA replacement.
% % The problem with simply replacing the mean teacher averaging in FixMatch with fast-SWA is performance degradation in the pure semi-SL setting.
% Instead, we hope to improve training dynamics through active learning.

% Our theorems and how our method can be used to help tranining instability.

\subsection{Algorithm}\label{sec:Algorithm}
Let $\lambda_{\min}\left(\scriptX\right)$ be the minimum eigenvalue of the symmetric Gram matrix $\scriptK_t(\scriptX, \scriptX)$.
We propose to encode the uncertainty about samples $\scriptX$ through the network's gradients $\nabla f_{\theta_t} (\scriptX)$ used to compute the Gram matrix $\scriptK_t$ and find the subset that solves
% value of labels corresponding to unlabeled samples $\scriptX_u$ by searching subsets in $\scriptX_U$ as
\begin{equation}\label{eq:objective}
    \max_{\scriptX_u \subset \scriptX_U} \min_{i \leq \abs{\scriptX_u \cup \scriptX_L}} \lambda_{i} \left(\scriptX_u \cup \scriptX_L\right) .
\end{equation}
% In Sec. \ref{sec:Analysis} we discuss how the above procedure enforces diversity and improves problem conditioning as advocated earlier.

\begin{algorithm*}[t]
\caption{Neural Pre-Conditioning (batch-mode solution to \eqref{eq:objective})}
\label{alg:batchCRC}
\begin{algorithmic}
    \STATE Inputs: Unlabeled pool $\scriptX_U$, acquisition size $Q$.
    \STATE Output: New pool $\scriptX_u^*$ to be labeled.
    \FOR{$i = 1, \cdots , m = \mathcal{O}\left(N_U\right)$}
        \STATE $\scriptX_u^{(i)} \gets Q$ unlabeled instances randomly sampled from $\scriptX_U$.
        \STATE $v(\scriptX_u^{(i)}) \gets \lambda_{\min}\left(\scriptX_L \cup \scriptX_u^{(i)}\right)$ using the network's Gram matrix.
    \ENDFOR
    \STATE Return $\mathcal{X}_u^* \gets \arg \max \left(v\right)$
\end{algorithmic}
\end{algorithm*}

% We now describe how uncertainty is considered through gradient directions and diversity is enforced.
To understand how our algorithm makes use of the direction of gradients to embed uncertainty information as argued in Sec. \ref{sec:Uncertainty}, consider the simple case of selecting one of two unlabeled samples $x_u, x_u'$.
Let $\scriptX = \{x_l, x_u, x_u'\}$ be the set containing these two and a labeled sample, and without loss of generality suppose $\abs{\nabla f(x_l)^T \nabla f(x_u)} = a > \abs{\nabla f(x_l)^T \nabla f(x_u')} = b$ with normalized gradients $\norm{\nabla f(x)} = 1$ for all $x \in \scriptX$.
Given potential labeled sets $X_1 = \{x_l, x_u\}$ and $X_2 = \{x_l, x_u'\}$, the minimum eigenvalues are $\lambda_{\min}(X_1) = 1 - a$ and $\lambda_{\min} (X_2) = 1 - b$.
Therefore, NPC measures the model's uncertainty about a sample $x_u'$ by comparing its gradient direction $\nabla f(x_u')$ with a more confident sample's $\nabla f(x_l)$ as a reference, ultimately returning a sample whose gradient direction is further away from the reference's to avoid querying for a less-informative sample.
% In summary, the NPC algorithm captures uncertainty information through the gradient's direction.

Next we address how the algorithm enforces diversity for batch-mode queries.
Consider multisets $\scriptX$, i.e. $\scriptX$ can have duplicate elements: $\scriptX \neq \scriptX \cup \{x\}$ for any $x \in \scriptX$.
A dataset $\scriptX$ with duplicate instances is called degenerate, or equivalently any non-degenerate set $\scriptX$ has elements $\norm{x_i - x_j} > 0$ for every pair $i\neq j$ indexing samples in $\scriptX$.
We show formally that NPC provably finds only non-degenerate solutions as long as such candidates exist.
Proposition \ref{lem:degenerate} alone resolves issues present in many AL algorithms that acquire identical samples on redundant datasets such as ``repeated MNIST'' \citep{batchbald}.

\begin{proposition}[NPC finds non-degenerate solutions]\label{lem:degenerate}
    Suppose $x_i \neq x_j \Rightarrow \scriptK^{(T)} (x_i, \cdot) \neq \scriptK^{(T)} (x_j, \cdot)$ for every $x_i, x_j \in \scriptX_L \cup \scriptX_U$.
    For any degenerate $\scriptX_u$ and non-degenerate $\scriptX_u^*$ sets,
    \begin{equation}
        \lambda_{\min}\left( \scriptX_L \cup \scriptX_u^*\right) 
            >
        \lambda_{\min} \left( \scriptX_L \cup \scriptX_u\right) = 0 .
    \end{equation}
\end{proposition}
\begin{remark}
    Intuitively, the assumption $x_i \neq x_j \Rightarrow \scriptK(x_i,\cdot) \neq \scriptK(x_j,\cdot)$ means that a high dimensional vector (function's gradients) is one-to-one on the small and countable domain $\scriptX_L \cup \scriptX_U$.
    This is true at least in the neighborhood of initialization for ReLU networks as long as not too many neurons are deactivated \citep{zhu19} or for another class of networks \cite{du19}.
\end{remark}

\begin{proof}
    The proof is a simple consequence of the rank-nullity theorem and positive definiteness.
    All eigenvalues computed over non-degenerate sets $\scriptX_L \cup \scriptX_u^*$ are non-zero since row vectors of $\hat{\scriptK}^{(t)}\left(\scriptX_L \cup \scriptX_u^*\right)$ are linearly independent. Because $\scriptK$ is semi-positive definite and singular only when its row vectors are linearly dependent, LHS $>0$.
    RHS has duplicate elements in the multiset, and therefore at least two row vectors are linearly dependent.
    Consequently $\scriptK^{(t)}$ is singular, implying RHS=0.
    % Furthermore, the result holds without loss of generality when the RHS's arguments are replaced with any degenerate candidate.
\end{proof}

One property that can be inferred from the above proposition is that NPC consolidates labeled data.
Interestingly, existing AL algorithms do not explicitly use labeled data when querying labels.
Because the labeled set at early acquisition steps may have been constructed using a semi-random acquisition step, or its measurements of uncertainty may have been unreliable because the network had been trained on such few labels, it is important that the label set is also re-evaluated against potential candidates so that label cost is not wasted on nearly-redundant samples' labels.
% If the initial labeled set contains a sample somewhat ``close'' to that in the unlabeled pool, or if measurements of uncertainty are unreliable because the 
% Networks trained using SGD generalize well partially because they do not overly-depend on each particular sample.
% Therefore, 

\subsection{Computational Considerations}\label{sec:Computation}
Computing the Gram matrix over a given candidate $\scriptX_u$ requires summing each layer's Gram matrix as $\scriptK_t = \sum_{l=1}^L \scriptK_t^{(l)}$.
Because each Gram matrix is semi-positive definite, its minimum eigenvalue is bounded below by the last layer's as $\lambda_{\min}\left(\scriptK_t^{(L)}\right) \geq \lambda_{\min}\left(\scriptK_t\right)$.
Therefore, we use only the last layer's gradients to compute $\scriptK_t$, where the resulting objective serves as a lower bound to Eq. \eqref{eq:objective}.
Furthermore we replace each block-element whose dimension is the number of classes with its trace to save memory.
% at the benefit of much faster speed.

When solving the inner-minimization, the kernel's value over labeled samples can be stored and re-used for every candidate batch $\scriptX_u$.
We compute the minimum eigenvalue using the robust and efficient locally optimal block preconditioned conjugate gradient method \citep{lobpcg}.
However, the search space of Eq. \eqref{eq:objective} is combinatorial in the pool size and query size.
Therefore we approximate the solution to Eq. \eqref{eq:objective} by sampling $m=\mathcal{O}\left(N_U\right)$ subsets uniformly at random to match the runtime of myopic algorithms, where $N_U$ is the unlabeled set's size.
By the inclusion-exclusion principle, the top $r N_U$ batches, with $r \in [0, 1]$, are included in the search space with probability $(1-r)^m$.
Taking $m = 1000$ as an example, the randomized search returns a batch within the 99-percentile with probability $\geq 1 - 4 \cdot 10^{-5}$.

\section{Discussion}\label{sec:Discussion}
\subsection{A Better Optimization Plateau for Generalization}
As motivated earlier, flattening out the landscape has positive implications towards generalization.
Here we prove that the landscape induced by labels acquired using NPC allows larger step sizes for convergence, which in turn leads to faster convergence towards flat landscapes.
At least for shallow 2-layer networks, increasing the convergence rate also reduces the generalization error \citep{Arora2019FineGrainedAO}.

For only this section, assume a non-degenerate training set: $\norm{x_i - x_j} > 0$ for each $i \neq j$.
% For only this section, we assume a non-degenerate training set $\norm{x_i - x_j} > 0, \forall i \neq j$.
% We first state our main theorem that motivates our algorithm.
\begin{theorem}\label{thm:main}
    At each gradient descent iteration $t$ with step size $\eta = \mathcal{O}(\lambda_{\min}\left(\scriptK_{0}\right))$, the MSE loss $\scriptL$ of a properly-initialized, sufficiently wide ReLU network decays as
    % A properly initialized feedforward ReLU network trained using gradient descent with step size $\eta = \mathcal{O}(\lambda_{\min}\left(\scriptK\right))$ on the MSE loss $\scriptL$ satisfies the following recursion
    \begin{equation}\label{eq:convergence}
        \scriptL_{t+1} \leq \left(1- \mathcal{O}\left(\eta \lambda_{\min} \left(\scriptK_t \right)\right) \right) \scriptL_t 
    \end{equation}    
    with high probability over initialization.
    % {\color{red}Each step or globally?}
\end{theorem}
Note that NTK-analyses typically express the training dynamics as a function of $\scriptK_0, \scriptK_{\infty},$ or the true NTK.
Although this can be done with additional perturbation analysis, we leave it at this form since we are concerned with the eigenvalue of the network's Gram matrix.

Two remarks follow.
First, the above shows that the set of step-sizes under which gradient descent converges is determined by $\lambda_{\min}\left(\scriptK_{0}\right)$.
The kernel $\scriptK_t$ essentially stays constant throughout training for a sufficiently wide network and is fixed as $\scriptK_{0}$ for simplicity.
Therefore, gradient descent can take large step-sizes and still converge when the labeled dataset is constructed using NPC.
% NPC can be viewed to increase the set of step-sizes that allow convergence.
% Furthermore, the result holds with high probability over random initialization.
% Convergence being possible for large step sizes indicates that the landscape is in a sense flat, increasing implicit regularization \citep{largeLearningRates}.
By maximizing $\lambda_{\min}\left(\scriptK_\infty\right)$, where $\scriptK_{\infty}$ is the Gram matrix of a classifier trained until near-convergence, NPC improves both training and generalization.
Second is the withstanding of Thm. \ref{thm:main} when the the computation of $\scriptK$ is reduced by using the last layer's Gram matrix. 
% Equation \eqref{eq:convergence} expresses the convergence rate of GD with respect to the kernel evaluated on a fixed labeled set.
% Therefore when evaluated on two different labeled sets, convergence rates provided by Eq. \eqref{eq:convergence} are different.
We described in Sec. \ref{sec:Computation} that our NPC algorithm solves Eq. \eqref{eq:objective} by replacing $\scriptK_t$ in $\lambda_{\min}(\scriptK_t)$ with the last layer's Gram matrix.
As shown, training and generalization benefits that come from solving Eq. \eqref{eq:objective} still hold when using the network's last layer to compute the kernel.

\subsection{Benefits of Randomized Search}
The alternation between querying for labels and training can be interpreted as a feedback system, which illustrates the exploration vs. exploitation effect of randomization used to solve Eq. \eqref{eq:objective} and complements the view that the network's uncertainty about samples is minimized with more labels.
% dynamics of enlarged datasets $\mathcal{Z}_L$ throughout iterations and effect of randomization used to solve Eq. \eqref{eq:objective}.
A state described by trained parameters $\theta$ and training set $\mathcal{Z}_L, \scriptX_U$ is used by a policy, which selects $\scriptX_u^*$ and acquires (observes) $\scriptY_u^*$.
At the first acquisition step, the network's gradients are unreliable measures of uncertainty embeddings due to a lack of labeled samples.
Subsequent states are then updated by propagating newly acquired labels to train the network so that gradient embeddings better represent uncertainty about samples.
% This updated state is then used to minimize the model's reliability by selecting the samples $\scriptX_u$ whose labels are estimated to be of most value.
At early acquisition stages, the randomized search therefore encourages the acquisition policy to explore instead of relying excessively on its belief.
% Earlier it appeared that randomization would only hurt the acquisition step; however, from this view we infer that randomized search serves as regularization to explore potentially better candidates.
The search space size decreases with more acquisition steps and therefore the policy progressively exploits its belief.
% In the limit that $N_U \to Q$, the ratio between the size of search space $N_U$ and that of candidates goes to $1$, and the policy progressively exploits information.

% 9/19/2022: Rephrase. What I'm trying to say: We are using an approximate kernle ridgeless regression. Bordelon shows that a training point reduces generalization error at modes with larger eigenvalues. We are essentially computing a "counterfactual" NTK, and evaluating its would-be resulting eigenvalues. If we can make many of its eigenvalues large, then the original training data may reduce generalization error even further. However we are approximating this "counterfactual NTK" with a finte-width network, and therefore it makes sense to not trust it so much by performing randomized search.

% Randomizing the search space can also help find data that helps generalization.
% A finite-width network's prediction is approximately kernel ridgeless regression \citep{Arora19}.
% \citet{bordelon2021spectrum} showed that a training point reduces generalization error at modes corresponding to larger eigenvalues.
% In this kernel regime, selecting a training set to increase eigenvalues therefore maximizes data efficiency in the sense that generalization error is most affected.
% Therefore, the same objective in Eq. \eqref{eq:objective} can be useful when performing AL for kernel regression.
% However given that we operate in the finite-width regime, randomization acts as a regularization.

Our algorithm's effect on generalization error can also be understood by studying the infinite-width regime, where we view randomization to act as a regularization method considering that we use finite-width networks.
A finite-width network's prediction is approximately kernel ridgeless regression \citep{Arora19}.
\citet{bordelon2021spectrum} showed that for kernel ridgeless regression, a training point reduces generalization error at modes corresponding to larger eigenvalues.
Our objective in Eq. \eqref{eq:objective} selects a training set that maximizes the minimum eigenvalue, and therefore enhances data efficiency in the sense that generalization error is affected at as many modes as possible.
However, finite-width ConvNets are generally better classifiers \citep{Arora2020Harnessing,Pennington} which were therefore used to query for labels.
The eigen-spectrum of a finite width network's Gram matrix is not identical to the NTK and solving the objective exactly may not directly translate to more generalization modes as for infinite width networks.

\section{Experiments}\label{sec:experiments}
% {\color{red}Analyze scores and how they are predictive/informative of the next training phase's properties.
% Start here.
% }

\subsection{Implementation Details}
We adopt all SSL-related configurations from \citep{OliverRealistic} and use the WRN-28-2 architecture \citep{wide_resnet} for all experiments.
At the first acquisition step, we randomly sampled $1$ image per class and used the model that attained median performance across 5 trials.
Subsequent acquisitions were performed with query size $Q = 20$ for CIFAR-10 and $Q=200$ for CIFAR-100.
All performances are averaged over 3 trials.
Following most AL setups, we train classifiers from scratch after each acquisition step.
Training from scratch better assesses the value of labels as it mitigates the possibility of vicious cycles where models trained sub-optimally in previous acquisition steps have no hope of improving despite superb data.

As discussed earlier, we assume no a-priori information on class distribution underlying unlabeled data.
Instead of estimating the underlying class distribution as done by \citet{darp} which may be detrimental given few labels, we simply take the target pseudo-label distribution to be uniform and perform pseudo-label refinement accordingly.

% Class imbalance is inevitable when labels are acquired incrementally.
% To resolve performance deterioration due to class imbalance, we equip FixMatch \citep{fixmatch} with DARP \citep{darp} where pseudo-labels are refined to match a target class distribution.
% In practice, the class distribution of the unlabeled pool is unknown, so we assign a uniform target class distribution to avoid using a-priori information.

\begin{table}[t]
\caption{CIFAR-10 $\left(Q_0 = 10, Q=20\right)$: Average accuracy (\%) $\pm$ standard deviation. 
Initial model achieved $57.64\%$ accuracy.
}
    \centering
        \begin{tabular}{c|ccc}
        \toprule
        % \multirow{*}{2}\# Labels \\ \multirow{*}{2}{}& 
        \backslashbox{Algorithm}{\# Labels} & 30 & 50 & 70 \\
        \midrule
        Passive & $78.08 \pm 5.48$ & $91.82 \pm 2.30$ & $91.00 \pm 2.78$ \\ 
        Margin & $87.36 \pm 5.01$ & $90.85 \pm 3.64$ & $90.85 \pm 3.64$ \\ 
        ALBL & $80.61 \pm 12.5$ & $89.62 \pm 6.66$ & $94.45\pm 0.20$ \\
        BADGE & $80.60 \pm 4.46$ & $86.43 \pm 1.22$ & $80.32 \pm 7.76$ \\
        NP$\text{C}^{\dagger}$   & $ 85.09 \pm 9.61$ & $\mathbf{94.63 \pm 0.07}$ & $\mathbf{94.85 \pm 0.02}$ 
        \\ \bottomrule
        \end{tabular}
\label{tab:cifar10}
% \vskip -0.1in
\end{table}

\subsection{Baseline Algorithms}\label{baseline}
% Both AL and SSL algorithms demand much more time than standard supervised learning, and an extensive comparison of all AL  is beyond our compute availability.
The proposed NPC algorithm is compared with passive learning where labels are queried uniformly at random, margin \citep{margin}, active learning by learning (ALBL, \citet{ALBL}) comprising least confidence and Coreset \citep{coreset}, and BADGE.
% BADGE was described earlier, 
Margin evaluates the classifier's margin and selects $Q$ samples whose margin is lowest.
ALBL employs a two-armed adversarial bandit algorithm to adapt to the better of least confidence $\arg \min_i f\left(x_i\right)$ and Coreset.
BADGE was described earlier, and acquires samples by applying k-means++ on gradient embeddings.
Entropy was also considered but excluded because of its low performance on some experiments.
% {\color{red}Coreset \citep{coreset} and ?}.
% BADGE is ??

% and two recently proposed AL algorithms that achieve state-of-the-art performance on standard AL benchmarks: Coreset \citep{coreset} and BADGE \citep{badge}.

% {\color{red}
% Coreset description; BADGE description.
% ALBL.
% Entropy failed to obtain meaningful results in CIFAR100.
% }

\subsection{Performance}

\begin{table*}[th]
\caption{CIFAR-100 $\left(Q_0 = 100, Q = 200\right)$: Average accuracy (\%) $\pm$ standard deviation. 
Initial model achieved $21.29\%$ accuracy.
% was that with median performance over $5$ trials at first acquisition ($Q_0$) which reached $21.29\%$ accuracy.
}
    \begin{tabular}{c|c c c c}
    \toprule
    \backslashbox{Algorithm}{\# Labels} & 300 & 500 & 700 & 900 \\
    \midrule
    Passive & $37.98 \pm 1.89$ & $47.11 \pm 1.41$ & $52.69 \pm 1.16$ & $57.97 \pm 1.43$  \\ 
    Margin & $37.41 \pm 1.22$ & $48.21 \pm 3.76$ & $52.53 \pm 0.75$ & $57.03 \pm 0.39$ \\ 
    ALBL & $39.05\pm 0.62$ & $\mathbf{49.88\pm 0.92}$ & $54.23 \pm 1.78$ & $56.65\pm 0.86$ \\
    BADGE & $24.55 \pm 1.13$ & $25.21\pm 1.86$ & $28.03 \pm 1.74$ & $29.81 \pm 1.72$ \\
    NP$\text{C}^{\dagger}$   & $\mathbf{40.92 \pm 1.65}$ & $48.16 \pm 0.77$ & $\mathbf{55.76 \pm 1.45}$ & $\mathbf{58.58 \pm 0.52}$ 
    \\ \bottomrule
    \end{tabular}
\label{tab:cifar100}
\end{table*}

% Compare algorithms for CIFAR-10 and CIFAR-100 using $Q = 20$. Leave $50$ for the next section.
Tables \ref{tab:cifar10} and \ref{tab:cifar100} show the accuracy of AL algorithms when trained on CIFAR-10 and CIFAR-100, respectively.
NPC outperforms other label acquisition schemes on nearly all dataset sizes and is at least competitive on the few others.
Although BADGE is state-of-the-art on AL benchmarks, we observe older algorithms performing better when evaluated by SSL accuracy.
This reveals how existing AL algorithms have been evaluated by their efficiency of sample acquisitions rather than label complexity.

To complement performances, Fig. \ref{fig:IoU} illustrates the similarity between algorithms as the intersection over union (IoU) of label indices as more labels are collected on CIFAR-100.
At any given label set size $N_L$, an algorithm's label set is the union of labels acquired at different trials.
Margin and least-confidence both rely heavily on the classifier's predictions, and labels acquired at different trials overlap significantly.
As shown, ALBL and Margin are similar in how labels acquired, which describes that using a classifier's least confidence is similar to acquiring based on its margin.
On the other hand, other pairs of algorithms have very small overlaps, demonstrating that their acquisition criteria are drastically different.

\begin{figure}[h] % [H]
    \centering
    \includegraphics[height=4cm,width=\textwidth]{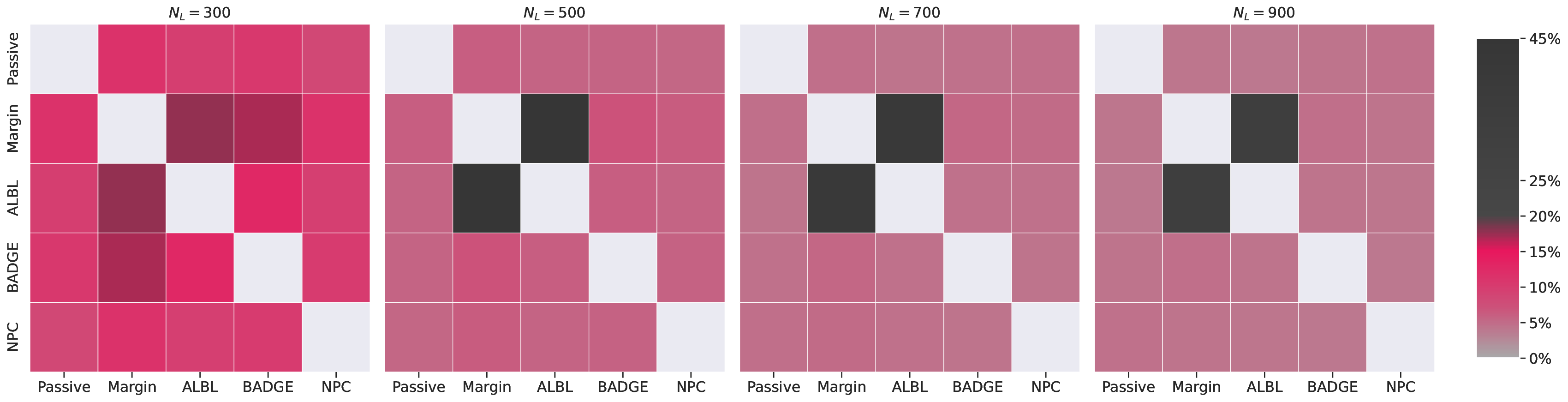} %0.45 \textwidth
    \caption{Similarity between AL algorithms: Intersection over union (IoU\%) of labels commonly acquired by algorithms.}
    \label{fig:IoU}
\end{figure}

% how many labels are commonly acquired by algorithms by visualizing intersection over union (IoU) of label indices on CIFAR-100 after each algorithm collects $700$ labels in Fig. \ref{fig:IoU}.
% samples acquired using BADGE is helpful as features $\scriptX$, but are not useful for their labels. 

% \begin{figure*}[t]
% \centering
% % Base: 1.5cm
% % \vspace{0pt}
% % \setlength{\extrarowheight}{0.5cm}
% % \setlength\tabcolsep{0cm} % default value: 6pt
% % \setlength{\tabcolsep{0.02cm}}
% % {\renewcommand{\arraystretch}{0.3}
% % \begin{tabular}{cccccccc}
% % \begin{tabular}{@{}c@{\hspace{0.0001\textwidth}}c@{\hspace{0.0001\textwidth}}c@{}}
% % % \begin{tabular}{@{}c@{\hspace{}}c@{\hspace{}}c@{}}
% %     \includegraphics[width = .33\textwidth]{figs/heatmap_round0.eps} &
% %     \includegraphics[width = .33\textwidth]{figs/heatmap_round1.eps} &
% %     \includegraphics[width = .33\textwidth]{figs/heatmap_round2.eps} \\
% % {$N_L = 300$}& {$N_L = 500$}& {$N_L = 700$} 
% % & {900} & {(e) Log-Euc} & {(f) AIM} & {(g) Log-Chol} & {(h) Style} 
% % \\
% % \end{tabular}
% % }
%     \caption{Intersection over union (IoU, \%) of labels commonly acquired by algorithms.}
% \label{fig:IoU}
% \end{figure*}

\subsection{Reducing Inquiry Frequency}\label{sec:zero_shot}
The above experiments aim to maximize accuracy with budget constraints on the number of labels.
Certain applications may additionally require that the number of inquiries is minimized to reduce the frequency of interaction between classifier and annotator.
It is clear that RANDOM remains unaffected by the number of inquiries, and it is desirable that AL algorithms maintain high performance when fewer inquiries are possible.

As observed in Tab. \ref{tab:cifar10}, NPC achieves very high performance at $N_L = 50$.
To accommodate a limit on the number of inquiries, we experiment with how NPC and BADGE, selected based on their similarity, are affected on what we call single and zero shot AL, both referring to a single query given a model trained on very few labels ($N_L = 10$) and a randomly initialized ($N_L = 0$) model, respectively. 
% Here we compare NPC with BADGE because of their similarity of valuating samples. 
As shown in Tab. \ref{tab:single_shot}, NPC excels in both zero and single shot settings where an imperfect classifier is used to valuate samples.
Interestingly, NPC remains nearly unaffected by which model is used to query for samples in zero or single shot queries.
In contrast, BADGE is detrimentally affected by its over-reliance on gradient embeddings on single-shot AL and rather performs better in zero-shot queries when gradients are randomly initialized.

% when minimizing the number of inquiries by performing single-shot $Q = B$ and zero-shot $Q_0 = 0, Q = B$ AL.
% The latter essentially demonstrates that NPC is useful at valuating labels even when using a random model, and is possible by virtue of NTK-landscape analysis near initialization.

\begin{table}[h]
% \centering
\caption{Zero-shot and Single-shot AL on CIFAR 10 using $Q \in \left\{30, 50\right\}$, where zero-shot refers to acquisition using a randomly initialized model and single-shot to a model trained on 1 label per class.
}\label{tab:single_shot}
\centering
    \begin{tabular}{ccccc}
        \toprule
         & \multicolumn{2}{c}{Zero Shot} & \multicolumn{2}{c}{Single Shot} \\
        \# Labels & 40 & 60 & 40 & 60\\
        % \multirow{2}{c}{Accuracy} & \multicolumn{2}{c}{Zero Shot} & \multicolumn{2}{c}{Single Shot} \\
        % & 40 & 60 & 40 & 60\\
        \midrule
% Zero-shot
        % Accuracy & 40 & 60 \\
        % \midrule
        BADGE & $92.13 \pm 3.24$ & $94.42 \pm 0.41$ & $86.70 \pm 5.42$ & $70.28 \pm 13.22$\\
        % Coreset & \\
        NP$\text{C}^{\dagger}$   & $93.14 \pm 1.42$ & $93.16 \pm 2.34$ & $92.26 \pm 3.26$ & $93.76 \pm 0.58$
% Single-shot
        % BADGE & $86.70 \pm 5.42$ & $70.28 \pm 13.22$\\
        % Coreset & \\
        % NP$\text{C}^{\dagger}$   & $92.26 \pm 3.26$ & $93.76 \pm 0.58$
        \\ \bottomrule
        \end{tabular}
% \vskip -0.1in
\end{table}

% ---------------------------------- %

% \begin{table}[h]
% % \centering
% \caption{a}\label{tab:single_shot}
% \centering
%     \begin{tabular}{ccc}
%         \toprule
%          & \multicolumn{2}{c}{Zero Shot}  \\
%         \# Labels & 0 \to 40 & 0 \to 60 \\
%         \midrule
%         BADGE & $92.13 \pm 3.24$ & $94.42 \pm 0.41$ \\
%         NP$\text{C}^{\dagger}$   & $93.14 \pm 1.42$ & $93.16 \pm 2.34$ \\
%         \bottomrule
%         \end{tabular}
% % \vskip -0.1in
% \end{table}

% \begin{table}[h]
% % \centering
% \caption{b
% }\label{tab:single_shot}
% \centering
%     \begin{tabular}{ccc}
%         \toprule
%          & \multicolumn{2}{c}{Single Shot}  \\
%         \# Labels & 10 \to 40 & 10 \to 60 \\
%         \midrule
%         BADGE & $86.70 \pm 5.42$ & $70.28 \pm 13.22$\\
%         NP$\text{C}^{\dagger}$ & $92.26 \pm 3.26$ & $93.76 \pm 0.58$ \\
%         \bottomrule
%         \end{tabular}
% \end{table}
% ---------------------------------- %

A few questions arise from this observation. 
The fact that BADGE performs worse on single-shot AL with a larger query size highlights that gradients alone may not be informative features in valuating samples as assumed.
% Labels collected from a random feature projection network (i.e. randomly initialized model) being more helpful in training a classifier than those collected by a trained pattern recognizer, it is likely that gradients by their selves may not be {\color{red}good features} as assumed \citep{badge,egl}.
NPC also uses gradients as features for valuating samples, but its robust performance with respect to number of queries can be attributed to our conclusion from theoretical analysis where problem conditioning is directly affected.

% {\color{red}FIX
% }
It's surprising how both BADGE and NPC perform extremely well on zero-shot AL, where a randomly initialized model decides which labels are most valuable.
For comparison, FixMatch on \emph{balanced data} without DARP reportedly achieves $86.19$\%, comparable to BADGE on single-shot but under-performing both BADGE and NPC on zero-shot.
% 13.81
To explain this phenomenon, it is instructive to view randomly initialized networks in their asymptotic limits.

At first glance it may appear that NPC with a randomly initialized network should not work well.
However, wide networks at initialization approximate their infinite-width NTK \citep{Arora19}.
As mentioned earlier, fully-trained wide networks are essentially ridge regression $\hat{y}_{ridge}\left(x_{test}\right) = \scriptK\left(x_{test}, X\right) \scriptK^{-1}\left(X, X\right) y$, and zero-shot NPC translates to a construction of the above kernel on which ridge regression will be performed.
A training point influences the generalization of kernel regression more for modes corresponding to large eigenvalues \citep{bordelon2021spectrum}, which is maximized by NPC.
In summary, NPC using a randomly initialized network selects samples to maximize generalization performance as predicted by approximate kernel regression through the NTK spectrum. 

\section{Conclusion}
This work motivated downstream SSL performance as a benchmark to evaluate AL algorithms.
We then described motivations recurrent in previous works and proposed an AL algorithm that addresses these concerns.
The proposed NPC algorithm captures uncertainty through the model's gradients, operates in the batch-mode setting, and improves the landscape of downstream SSL through data acquisition as measured by properties related to generalization.
% theoretically analyzed gradient descent dynamics to derive a kernel-dependent bound.
% theoretically analyzing how different datasets affect the optimization of DNNs and how the two affect each other.
% This kernel-dependent bound was then exploited to design an AL algorithm we call NPC that queries for labels that best conditions downstream training.
% By relating to known limitations of semi-supervised learning, kernel-ridge regression, and analyzing how NPC un-correlates features, intuitions underlying why performance should be enhanced were discussed.
% and demonstrated that querying unlabeled data based on the resultant problem condition significantly improves label efficiency.
% To the best of our knowledge, NPC is the first practical algorithmic application of NTK outside of its use in theoretical analysis and small classification tasks \citep{Arora2020Harnessing}.
% In contrast to most DL-based AL experiments which use relatively outdated small networks, this work demonstrates that DL can benefit from AL using modern architectures.
Experiments re-evaluating state-of-the-art AL algorithms with respect to downstream SSL performance, which better measures label complexity, demonstrate that NPC outperforms other AL algorithms on most dataset sizes and tasks.
% is at least competitive to all other algorithms.
% outperformed uncertainty-based AL algorithms and another algorithm motivated by a similar goal of maximizing the model's change.
% Notably, the proposed algorithm's superiority was shown in the imperative batch-AL setting.

NPC enjoys several properties that aren't present in other AL algorithms or is at least not obvious.
First, NPC explicitly consolidates existing labeled data when measuring the value of labeling unlabeled candidates.
Moreover, NPC is a batch AL algorithm that provably selects distinct samples.
The proposed algorithm is also interesting in that it is a kernel-based sampling scheme.
Kernels are excellent models of data distributions, and NPC's construction of a kernel using DL opens new venues for AL.
% This work can be extended in several directions.
% Our choice of using both AL and SSL is for a more informative evaluation and NPC either a-priori affects the downstream SSL performance or is affected (reduced width) after training with the use of SSL.
% Instead, SSL can be used to complement data selection as in \citep{CEAL} or conversely AL can be incorporated when designing SSL algorithms.
% For example, NPC provides an estimate on the number of epochs required for convergence, and a learning rate schedule best fit for the estimated time horizon could be used \citep{learning_schedules}.
% A preliminary demonstration is provided in the Appendix where we show how well NPC's score reflects the time horizon til convergence.
% Lastly, AL algorithms could be combined into one AL algorithm similar to how FixMatch modifies a combination of several seminal SSL algorithms.

A few limitations and future works are described.
Our experiments rely on modern SSL algorithms to evaluate AL algorithms.
Although current SSL algorithms achieve extremely high accuracy on vision tasks, they suffer from algorithmic instability where given the same model and dataset, their performances vary more-so than supervised learning.
Ideally, all algorithms should achieve higher accuracy in line with the ``more data is better'' principle.
Because performance deterred by class imbalance is resolved using pseudo label refinements, we believe experimental evaluations will benefit most from algorithmic stability. 
Further, SSL training demands much more computation than SL counterparts, and consequently an exhaustive evaluation of various AL algorithms is prohibitive.
Experimental protocols that reduce computations in evaluating algorithms yet are fair would expedite research.
Lastly, we treat AL and SSL phases independently for our purpose.
An interesting direction to pursue would be to design AL and SSL schemes that adapt to each other.
For example, our theoretical analysis and NPC's valuation gives an upper bound on possible learning rates for downstream training.
By designing learning schedules to adapt to the set of admissable step sizes, downstream training may be better stabilized and achieve higher performance.

\begin{ack}
    We thank Professor R. Srikant for helpful discussions relating to Neural Tangent Kernels.
\end{ack}

\bibliography{example_paper}
\bibliographystyle{icml2022}

%%%%%%%%%%%%%%%%%%%%%%%%%%%%%%%%%%%%%%%%%%%%%%%%%%%%%%%%%%%%%%%%%%%%%%%%%%%%%%%
%%%%%%%%%%%%%%%%%%%%%%%%%%%%%%%%%%%%%%%%%%%%%%%%%%%%%%%%%%%%%%%%%%%%%%%%%%%%%%%
% APPENDIX
%%%%%%%%%%%%%%%%%%%%%%%%%%%%%%%%%%%%%%%%%%%%%%%%%%%%%%%%%%%%%%%%%%%%%%%%%%%%%%%
%%%%%%%%%%%%%%%%%%%%%%%%%%%%%%%%%%%%%%%%%%%%%%%%%%%%%%%%%%%%%%%%%%%%%%%%%%%%%%%

\section*{Checklist}

%%% BEGIN INSTRUCTIONS %%%
% The checklist follows the references.  Please
% read the checklist guidelines carefully for information on how to answer these
% questions.  For each question, change the default \answerTODO{} to \answerYes{},
% \answerNo{}, or \answerNA{}.  You are strongly encouraged to include a {\bf
% justification to your answer}, either by referencing the appropriate section of
% your paper or providing a brief inline description.  For example:
% \begin{itemize}
%   \item Did you include the license to the code and datasets? \answerYes{See Section~\ref{gen_inst}.}
%   \item Did you include the license to the code and datasets? \answerNo{The code and the data are proprietary.}
%   \item Did you include the license to the code and datasets? \answerNA{}
% \end{itemize}
% Please do not modify the questions and only use the provided macros for your
% answers.  Note that the Checklist section does not count towards the page
% limit.  In your paper, please delete this instructions block and only keep the
% Checklist section heading above along with the questions/answers below.
%%% END INSTRUCTIONS %%%

\begin{enumerate}

\item For all authors...
\begin{enumerate}
  \item Do the main claims made in the abstract and introduction accurately reflect the paper's contributions and scope?
    % \answerTODO{}
    \answerYes
  \item Did you describe the limitations of your work?
    % \answerTODO{}
    \answerYes
  \item Did you discuss any potential negative societal impacts of your work?
    \answerNA{}
  \item Have you read the ethics review guidelines and ensured that your paper conforms to them?
    \answerYes{}
\end{enumerate}

\item If you are including theoretical results...
\begin{enumerate}
  \item Did you state the full set of assumptions of all theoretical results?
    \answerNo{We defer the precise arguments to the Appendix. The theorem serves as an answer to our motivations, and we believe the precise assumptions deter from the main message.}
        \item Did you include complete proofs of all theoretical results?
    \answerYes{Included in Appendix.}
\end{enumerate}

\item If you ran experiments...
\begin{enumerate}
  \item Did you include the code, data, and instructions needed to reproduce the main experimental results (either in the supplemental material or as a URL)?
    \answerYes{}
  \item Did you specify all the training details (e.g., data splits, hyperparameters, how they were chosen)?
    \answerYes{}
        \item Did you report error bars (e.g., with respect to the random seed after running experiments multiple times)?
    \answerYes{}
        \item Did you include the total amount of compute and the type of resources used (e.g., type of GPUs, internal cluster, or cloud provider)?
    \answerNo{Computational requirements follow standard semi-supervised learning settings.}
\end{enumerate}

\item If you are using existing assets (e.g., code, data, models) or curating/releasing new assets...
\begin{enumerate}
  \item If your work uses existing assets, did you cite the creators?
    \answerYes{}
  \item Did you mention the license of the assets?
    \answerNo{Data and base code used for experiments are described.}
  \item Did you include any new assets either in the supplemental material or as a URL?
    \answerNA{}
  \item Did you discuss whether and how consent was obtained from people whose data you're using/curating?
    \answerNA{}
  \item Did you discuss whether the data you are using/curating contains personally identifiable information or offensive content?
    \answerNA{}
\end{enumerate}

\item If you used crowdsourcing or conducted research with human subjects...
\begin{enumerate}
  \item Did you include the full text of instructions given to participants and screenshots, if applicable?
    \answerNA{}
  \item Did you describe any potential participant risks, with links to Institutional Review Board (IRB) approvals, if applicable?
    \answerNA{}
  \item Did you include the estimated hourly wage paid to participants and the total amount spent on participant compensation?
    \answerNA{}
\end{enumerate}

\end{enumerate}

\appendix

\section{Proof of Theorem 1}
Assume a non-degenerate training set $\norm{x_i - x_j} > 0, \forall i \neq j$.
Theorem 1 in the main script is re-written:
\begin{theorem}\label{thm:main}
    At each gradient descent iteration $t$ with step size $\eta = \mathcal{O}(\lambda_{\min}\left(\scriptK_0\right))$, the MSE loss $\scriptL$ suffered by a properly-initialized feedforward ReLU network decays as
    % A properly initialized feedforward ReLU network trained using gradient descent with step size $\eta = \mathcal{O}(\lambda_{\min}\left(\scriptK\right))$ on the MSE loss $\scriptL$ satisfies the following recursion
    \begin{equation}
        \scriptL_{t+1} \leq \left(1- \mathcal{O}\left(\eta \lambda_{\min} \left(\scriptK_t \right)\right) \right) \scriptL_t 
    \end{equation} 
    with high probability over initialization.
\end{theorem}

We adopt the convention that all gradients are flattened in vector form and use the Euclidean norms to represent their size.
First we express training dynamics as a recursion:
% By the fundamental theorem of calculus, the following holds:
\begin{lemma}\label{lem:du}
        Feedforward DNNs with once-differentiable activation functions trained using gradient descent on the MSE loss $\scriptL_t$ with step size $\eta$ follows the recursion:
    \begin{equation}\label{eq:lemma}
        \scriptL_{t+1} \leq \left(1-\eta \lambda_{\min}\left(\scriptK_t\right)\right) \scriptL_t + \xi_t + \epsilon_t
        ,
    \end{equation}
    where $\xi_t
        =
        \int_{0}^{\eta} \nabla \scriptL_{t}^T \left(\nabla \scriptL_{t} - \nabla \scriptL(\theta_t - \gamma \nabla \scriptL_t )\right)d\gamma$ and
        $\epsilon_t =
        \frac{1}{2}(f_{\theta_{t+1}} - f_{\theta_t})^2$.
\end{lemma}
\begin{proof}
This derivation is mostly from \cite{du19}, but we include the proof under our notations for completeness. Let $e_t = y - f_{\theta_t}$.
A standard technique with triangular inequality gives
\begin{equation}\label{eq:basic}
    \scriptL_{t+1} \leq \scriptL_t + \norm{f_{\theta_{t+1}} - f_{\theta_t}}^2 - 2 e_t^T \left(f_{\theta_{t+1}} - f_{\theta_t}\right).
\end{equation}
Let $h(\eta) = f(\theta_t - \eta \nabla \scriptL_t)$. 
By the fundamental theorem of calculus,
\begin{align*}
    f_{\theta_{t+1}} - f(\theta_t) 
        =
    h(\eta) - h(0) \\
    = \int_0^{\eta} h'(\gamma) d\gamma
    = \int_0^{\eta} h'(0) d\gamma + \int_0^{\eta} h'(\gamma) - h'(0) d\gamma
\end{align*}
Since $h'(0) = - \nabla f(\theta_t)^T \nabla \scriptL_t = - e \nabla f_{\theta_t}^T \nabla f_{\theta_t}= - e \text{Tr}\left(\scriptK_t\right)$, we have
\begin{align*}
    e^T (f_{\theta_{t+1}} - f_{\theta_t}) = - \eta e^T \scriptK_t e + \int_0^{\eta} h'(\gamma) - h'(0)d \gamma
    \leq -\eta \lambda_{\min}\left(\scriptK_t\right) \scriptL_t  + \xi_t.
\end{align*}
Substituting into Eq. \ref{eq:basic} gives Eq. \ref{eq:lemma} together with $e_t \int_0^{\eta} h'(\gamma) - h'(0)d\gamma = \int_0^{\eta} \nabla \scriptL_t^T \left(\nabla \scriptL_t - \nabla \scriptL(\theta_t - \gamma \nabla \scriptL_t\right)d\gamma$.    
\end{proof}

The above bound sheds light on training dynamics, where the first term decreases linearly with rate determined by the Gram matrix' eigenvalue.
To establish Thm. \ref{thm:main} that states the loss descends at each gradient step, it remains to prove that residual terms $\xi_t, \epsilon_t$ grow (sub-)linearly with $\scriptL_t$.

An extension of smoothness and convexity is defined following \citep{zhu19}:
\begin{definition}[Smoothness]
    A non-negative, once-differentiable function $g\in C^1 (\scriptX)$ is $(\alpha, \beta)$-smooth if for every $x, y \in \scriptX$,
    \begin{align*}
        g(y) \leq g(x) + \nabla g(x)^T (y-x) + \alpha \sqrt{g(x)}\norm{y-x} + \beta \norm{y-x}^2 \numberthis
    \end{align*}
    % for every $x,y \in \scriptX$.
\end{definition}
% Another useful definition resembling convexity:
\begin{definition}[Near-Convexity]
    A non-negative function $g \in C^1 (\scriptX)$ has gradients $\nabla g$ that scale as $(\mu, M)$ if
    \begin{equation}
        \mu g(x) \leq \norm{\nabla g(x)}^2 \leq M g(x), \forall x \in \scriptX.
    \end{equation}
    If a function's gradients scale as $(\mu, M)$, we say the gradient scale is bounded.
\end{definition}

First we invoke the following lemma (Thms. 3 \& 4 in \citet{zhu19}) to show that the MSE loss remains semi-smooth and nearly convex throughout training for wide ReLU networks:
\begin{lemma}
    For sufficiently small $\norm{\theta - \theta_0}$ and $\norm{\theta - \theta'}$, the loss remains nearly convex 
    \begin{align*}
        \norm{\nabla \mathcal{L}\left(\theta\right)}^2 = \Theta\left(\mathcal{L}\left(\theta\right)\right) 
    \end{align*}
    and semi-smooth
    \begin{align*}
        \mathcal{L}\left(\theta'\right) 
            \leq 
        \mathcal{L}\left(\theta\right) 
            + \nabla \mathcal{L} \left(\theta\right) \left(\theta' - \theta\right) + \mathcal{O}\left(\scriptL\left(\theta\right)^{1/2} \norm{\theta' - \theta}\right) 
            + \mathcal{O}\left(\norm{\theta' - \theta}^2\right)
    \end{align*}
    with high probability hiding constants depending on architecture width, depth, and dataset size.
\end{lemma}
Above we use $\Theta \left(\cdot\right)$ as upper and lower bounds matching up to multiplicative constants.

Next we bound the residual terms in Lemma \ref{lem:du}:
\begin{lemma}\label{lem:main_lemma}
    % Properly initialized and sufficiently Wide RELU networks trained using GD gives
    If the loss function $\mathcal{L}_t$ remains smooth and near-convex as defined above,
    \begin{align*}
        \epsilon_t, \xi_t \leq \mathcal{O}(\eta^2) \scriptL_t 
    \end{align*}
    with high probability over initialization.
\end{lemma}
% we can invoke the following lemma (Thms. 3 \& 4 in \citet{zhu19}) to show that the MSE loss remains semi-smooth and nearly convex throughout training for wide ReLU networks:
% \begin{lemma}
%     For sufficiently small $\norm{\theta - \theta_0}$ and $\norm{\theta - \theta'}$, the loss remains nearly convex 
%     \begin{align*}
%         \norm{\nabla \mathcal{L}\left(\theta\right)}^2 = \Theta\left(\mathcal{L}\left(\theta\right)\right) 
%     \end{align*}
%     and semi-smooth
%     \begin{align*}
%         \mathcal{L}\left(\theta'\right) & \leq 
%             \mathcal{L}\left(\theta\right) 
%             + \nabla \mathcal{L} \left(\theta\right) \left(\theta' - \theta\right) 
%             \\ & + \mathcal{O}\left(\scriptL\left(\theta\right)^{1/2} \norm{\theta' - \theta}\right) 
%             + \mathcal{O}\left(\norm{\theta' - \theta}^2\right)
%     \end{align*}
%     with high probability hiding constants depending on architecture width, depth, and dataset size.
% \end{lemma}
% Above we use $\Theta \left(\cdot\right)$ as upper and lower bounds matching up to multiplicative constants.

\begin{proof}  
    The following inequality will be used for $(\alpha, \beta)$-smooth functions.
    \begin{proposition}\label{prop:smooth}
        If $g$ is $(\alpha, \beta)$-smooth,
        % \begin{align*}
        \begin{equation}
            (\nabla g(y) - \nabla g(x)) (y-x) \leq \alpha (\sqrt{g(x)} + \sqrt{g(y)}) \norm{y-x} + 2\beta \norm{y-x}^2 
        \end{equation}
        % \numberthis 
        % \end{align*}
    \end{proposition}
    \begin{proof}
    Expanding the LHS in terms of $x$ and $y$ then summing their upper bounds gives the inequality.
    % Proof of second property follows from expanding LHS in terms of $x$ and $y$, then summing their upper bounds.
    \end{proof}

    \textbf{Bound on $\xi_t$}
    Proposition \ref{prop:smooth} with $\scriptL$ at $\theta_t$ and $\theta_t - \gamma \nabla \scriptL_t$ can be used to bound the integrand.
    \begin{align*}
        \left(\nabla \scriptL_t - \nabla \scriptL (\theta_t - \gamma \nabla \scriptL_t)\right) \nabla \scriptL_t & 
        % \\ 
            \leq 
         \alpha \norm{\nabla \scriptL_t} \left(\sqrt{\scriptL_t} + \sqrt{\scriptL(\theta_t - \gamma \nabla \scriptL_t)}\right) + 2 \gamma \beta \norm{\nabla \scriptL_t}^2 &
    .\end{align*}
    Using the definition of smoothness 
    \begin{align*}
        \scriptL(\theta_t - \gamma \nabla \scriptL_t) & 
        % \\
             \leq
        \scriptL_t 
        + \gamma \left(\alpha \sqrt{\scriptL_t} \norm{\nabla \scriptL_t} - \norm{\nabla \scriptL_t}^2\right)
        + \beta \gamma^2 \norm{\nabla \scriptL_t}^2
        & ,
    \end{align*}    
    and by near-convexity, 
    \begin{equation}\label{eq:smooth_drift}
        \leq \left(1 + \gamma (\alpha \sqrt{M}-\mu) + \beta \gamma^2\right)\scriptL_t.
    \end{equation}
    Let $b = \left(\alpha \sqrt{M} - \mu \right)/2\beta$ and $c= 1/\beta - b^2$.
    \begin{align*}
    \sqrt{\scriptL_t} + \sqrt{\scriptL(\theta_t - \gamma \nabla \scriptL_t)} 
        % \\
        \leq 
    % \sqrt{\beta \scriptL_t}\left(\gamma + \abs{b} + \sqrt{\abs{c}}\right) =: \sqrt{\beta \scriptL_t}\left(\gamma + c'\right)
    \sqrt{\scriptL_t}\left(1 + \sqrt{\beta} \left(\gamma + \abs{b} + \sqrt{\abs{c}}\right)\right) =: \sqrt{ \scriptL_t}\left(\sqrt{\beta}\gamma + c'\right)
    \end{align*}
    by the triangle inequality. 
    Again, $\norm{\nabla \scriptL_t}^2 \leq M \scriptL_t$, and we have a bound on the integrand as
    \begin{align*}
        \alpha \norm{\nabla \scriptL_t} \left(\sqrt{\scriptL_t} + \sqrt{\scriptL(\theta_t - \gamma \nabla \scriptL_t)}\right) + 2 \gamma \beta \norm{\nabla \scriptL_t}^2 
    % \\
    &
        \leq \left(\alpha \sqrt{M} \left(\sqrt{\beta}\gamma + c'\right) + 2 \gamma \beta M \right) \scriptL_t 
    \\
    &
        =: \left(a' \gamma + c''\right) \scriptL_t
    \\
    \Rightarrow
        \xi_t \leq  \scriptL_t \int_{0}^{\eta} a' \gamma + c'' d\gamma 
        % \\ 
        & 
    = O\left(\eta^2 \right) \scriptL_t.
    \end{align*}    
    where we hide constants that depend on the architecture and dataset size.

    {\textbf{Bound on $\epsilon_t$}}    
    It is sufficient that $\epsilon_t \leq \left(a \eta^2  + \lambda_{\min} \eta\right)\scriptL_t$ for any $a$ so that $\scriptL_t$ is guaranteed to decrease for small $\eta$. 
    This proof is quite involved and relies on analytic expressions for ReLU networks. 
    To this end, we follow the setting in \cite{zhu19} and WLOG fix the last layer's weights as $B$, denoting pre- and post- activations by $g^l, h^l$ respectively and an `active-indicator' matrix $D^{l} \in \mathbb{R}^{d \times d}$, $D^{l}_{k,k}= \mathbf{1}\left\{g^l_{k,k} \geq 0\right\}$, and weight matrices $W_l \in \mathbb{R}^{d\times d}$ for each layer $l \in [L]$, where $d$ denotes the width of the hidden layers and $L$ is the number of layers.

    Notice that for ReLU networks, we can write the post-activations at every layer as $h_{t+1}^l - h_t^l = D_{t+1}^l W_{t+1} h_{t+1}^{l-1} - D_t^l W_t^l h_t^{l-1}$.
    \begin{proposition}[Distributive diagonal matrices]\label{prop:linear}
        There exists $\tilde{D} = \left(\tilde{D}^{1}, \dots, \tilde{D}^{L}\right)$ with $\tilde{D}^{l} \in [-1,1]^{d\times d}$ for every $l$ such that
        \begin{align*}\label{eq:diagonal}
            D_{t+1}^l W_{t+1}^l h_{t+1}^l - D_t^l W_t^l h^{l-1}_t 
            % \\ 
            = \left(D_t^l + \tilde{D}^{l}\right)\left(W_{t+1}^l h_{t+1}^{l-1} - W_t^l h_t^{l-1}\right).
        \end{align*}
    \end{proposition}
    The above proposition follows from case-by-case considerations of ReLU activations, see Proposition 11.3 in \citet{zhu19}.
    
    \begin{proposition}[Linear expansion of post-activations]\label{prop:zhu}
    There exists some $\tilde{D}^l \in [-1,1]^{d \times d}$ at each $l$ such that
        \begin{align*}
            h^{l}_{t+1} - h^l_t = 
                - \eta \sum_{r=1}^l \left(D^l_{t} + \tilde{D}^l\right) W^l_t \cdots W^{r+1}_t \left(D^r_t + \tilde{D}^r\right) 
            % \\ 
            \times
                \left(\nabla_{W^r_t} \scriptL_t \right) h_{t+1}^{r-1}
                % , \forall l \in [L].
        \end{align*}
    \end{proposition}
    The following proposition due to \citet{zhu19} (Lemma 8.6b and Lemma 7.1, respectively) gives bounds on the first line on the RHS and last term:
    \begin{proposition}
        For every $l \in [L]$ and $r \in [l]$,
        \begin{align*}
            \norm{\left(D^l_{t} + \tilde{D}^l\right) W^l_t \cdots W^{r+1}_t \left(D^r_t + \tilde{D}^r\right)}
            \leq  O(\sqrt{L}) 
            % \\
            \norm{h_{t+1}^{r-1}} \leq o(1).
        \end{align*}
    \end{proposition}
    Applying Cauchy-Schwartz inequality and the fact that norm of sums $\leq$ sum of norms to Propositions \ref{prop:linear} and \ref{prop:zhu}, 
    \begin{align*}
        \norm{f_{\theta_{t+1}} - f_{\theta_t}} = \norm{B\left(h^L_{t+1} - h^L_t\right)},
            % \\
        \leq \eta O(L^{1.5} \sqrt{d}) \norm{\nabla \scriptL_t}.
    \end{align*}
    Since $\norm{\nabla \scriptL_t} \leq \sqrt{M \scriptL_t}$, 
    \begin{equation}
        \epsilon_t = \norm{f_{\theta_{t+1}}- f_{\theta_t}}^2 \leq O(L^3 d M) \eta^2 \scriptL_t = O(\eta^2) \scriptL_t.
    \end{equation}    
\end{proof}

Theorem \ref{thm:main} is a direct consequence of Lemmas \ref{lem:du} and \ref{lem:main_lemma}, and the step-size can be selected based on $\scriptK_0$ because $\scriptK_t$ remains in a neighborhood of $\scriptK_0$ throughout training \citep{Arora19}.

\end{document}